\documentclass{article}
\PassOptionsToPackage{finalizecache=false,frozencache=true}{minted}
\usepackage{arxiv}

\usepackage[utf8]{inputenc} 
\usepackage[T1]{fontenc}    
\usepackage{hyperref}       
\usepackage{url}            
\usepackage{booktabs}       
\usepackage{amsfonts}       
\usepackage{nicefrac}       
\usepackage{microtype}      
\usepackage{xcolor}         
\usepackage{graphicx}
\usepackage{subcaption}
\usepackage{booktabs}
\usepackage{tikz-network}
\usepackage{natbib}

\usepackage{amsmath}
\usepackage{amssymb}
\usepackage{mathtools}
\usepackage{amsthm}
\usepackage{bm}
\usepackage[noframes,nobars]{ffcode}
\usepackage{nicefrac}
\usepackage{siunitx}
\usepackage[capitalize,noabbrev]{cleveref}

\usepackage{thmtools, thm-restate}

\usepackage[acronym]{glossaries}
\usepackage{glossaries-extra}
\setabbreviationstyle{long-short-sc}

\newabbreviation[
  longplural={Gaussian Processes}
]{gp}{gp}{Gaussian Process}
\newabbreviation[
]{clt}{clt}{Central Limit Theorem}
\newabbreviation[
  longplural={Hilbert--space Gaussian Processes}
]{hgp}{hgp}{Hilbert--space Gaussian Process}
\newabbreviation[
  longplural={Tensor Networks}
]{tn}{tn}{Tensor Network}
\newabbreviation[
]{cpd}{cpd}{Canonical Polyadic Decomposition}
\newabbreviation[
]{tt}{tt}{Tensor Train}
\newabbreviation[
]{bf}{bf}{Basis Function}
\newabbreviation[
]{ml}{ml}{maximum likelihood}
\newabbreviation[
]{map}{map}{maximum a posteriori}
\newabbreviation[
]{dnn}{dnn}{Deep Neural Network}
\newabbreviation[
]{cnn}{cnn}{Convolutional Neural Network}
\newabbreviation[
]{rnn}{rnn}{Recurrent Neural Network}
\newabbreviation[
]{gnn}{gnn}{Graph Neural Network}
\newabbreviation[
]{mera}{mera}{Multi-Scale Entanglement Renormalization Ansatz}
\newabbreviation[
]{pdf}{pdf}{Probability Density Function}
\newabbreviation[
]{cdf}{cdf}{Cumulative Density Function}
\newabbreviation[
]{rmse}{rmse}{Root Mean Squared Error}
\newabbreviation[
]{krr}{krr}{Kernel Ridge Regression}
\newabbreviation[
]{rr}{rr}{Ridge Regression}
\newabbreviation[
]{lr}{lr}{Lasso Regression}
\newabbreviation[
]{als}{als}{Alternating Least-Squares}
\newabbreviation[
]{uci}{uci}{University of California, Irvine}
\newabbreviation[
]{kkt}{kkt}{Karush-Kuhn-Tucker}
\newabbreviation[
]{mlsvd}{mlsvd}{Multilinear Singular Value Decomposition}
\newabbreviation[
]{kmlsvd}{kmlsvd}{Kernel Multilinear Singular Value Decomposition}
\newabbreviation[
]{hosvd}{hosvd}{Higher-Order Singular Value Decomposition}
\newabbreviation[
]{svd}{svd}{Singular Value Decomposition}
\newabbreviation[
]{ksvd}{ksvd}{Kernel Singular Value Decomposition}
\newabbreviation[
]{pca}{pca}{Principal Component Analysis}
\newabbreviation[
]{kpca}{kpca}{Kernel Principal Component Analysis}
\newabbreviation[
]{svm}{svm}{Support Vector Machine}
\newabbreviation[
]{ls-svm}{ls-svm}{Least-Squares Support Vector Machine}
\newabbreviation[
]{rff}{rff}{Random Fourier Features}
\theoremstyle{plain}
\newtheorem{theorem}{Theorem}[section]

\newtheorem{corollary}[theorem]{Corollary}
\theoremstyle{definition}
\newtheorem{definition}[theorem]{Definition}

\theoremstyle{remark}
\newtheorem{remark}[theorem]{Remark}
\newtheorem{example}[theorem]{Example}

\crefname{equation}{equation}{equations}
\Crefname{equation}{Equation}{Equations}
\crefname{theorem}{theorem}{theorem}
\Crefname{theorem}{Theorem}{Theorems}
\crefname{definition}{definition}{definitions}
\crefname{corollary}{corollary}{corollaries}
\Crefname{corollary}{Corollary}{Corollaries}
\crefname{lemma}{lemma}{lemmas}
\Crefname{lemma}{Lemma}{Lemmas}
\crefname{appendix}{appendix}{appendices}
\crefname{section}{section}{sections}
\crefname{figure}{figure}{figures}

\newcommand{\mat}[1]{\bm{#1}}
\newcommand{\ten}[1]{\bm{\mathcal{#1}}}

\newcommand{\vectorize}[1]{\operatorname{\ff{vec}}\left(#1\right)}

\newcommand{\kron}{\otimes}

\newcommand{\trans}{^\mathrm{T}}

\newcommand{\pinv}{^\dagger}
\newcommand{\inv}{^{-1}}
\DeclareMathOperator{\Tr}{Tr}

\DeclareMathOperator{\isdef}{\coloneqq}
\DeclareMathOperator{\defis}{\reflectbox{$\coloneqq$}}

\title{A Kernelizable Primal-Dual Formulation of the Multilinear Singular Value Decomposition}

\date{}

\newif\ifuniqueAffiliation
\uniqueAffiliationtrue

\ifuniqueAffiliation 
\author{Frederiek Wesel\\
    Delft Center for Systems and Control\\
	Delft University of Technology\\
	The Netherlands\\
	\texttt{f.wesel@tudelft.nl} \\
	\And
	Kim Batselier \\
    Delft Center for Systems and Control\\
	Delft University of Technology\\
	The Netherlands\\
	\texttt{k.batselier@tudelft.nl} \\
}

\begin{document}
\maketitle

\begin{abstract}
    The ability to express a learning task in terms of a primal and a dual optimization problem lies at the core of a plethora of machine learning methods. 
    For example, \gls*{svm}, \gls*{ls-svm}, \gls*{rr}, \gls*{lr}, \gls*{pca}, and more recently \gls*{svd} have all been defined either in terms of primal weights or in terms of dual Lagrange multipliers. 
    The primal formulation is computationally advantageous in the case of large sample size while the dual is preferred for high-dimensional data. Crucially, said learning
    problems can be made nonlinear through the introduction of a feature map in the primal problem, which corresponds to applying the kernel trick in the dual. In this paper we derive a primal-dual formulation of the \gls*{mlsvd}, which recovers as special cases both \gls*{pca} and \gls*{svd}. Besides enabling computational gains through the derived primal formulation, we propose a nonlinear extension of the \gls*{mlsvd} using feature maps, which results in a dual problem where a kernel tensor
    arises. We discuss potential applications in the context of signal analysis and deep learning.
\end{abstract}

\keywords{\emph{Multilinear Singular Value Decomposition} (\gls*{mlsvd}) \and Tucker Decomposition \and Tensor  Decompositions \and Tensor Networks \and Primal-Dual \and Kernel Machines}
\section{Introduction}

The linear \emph{Support Vector Machine} (\gls*{svm}) was for the first time extended to the nonlinear case by \citet{boser_training_1992}, giving rise to modern \gls*{svm} theory. 
Central to this extension is the primal-dual formulation of the learning problem, which allows for a nonlinear extension in terms of a primal feature map $\mat{\phi}(\cdot):\mathbb{R}^{N_1}\rightarrow \mathbb{R}^{M}$ that maps an inputs $\mat{x} \in \mathbb{R}^{N_1}$ to a higher (possibly infinite) dimensional space $\mathbb{R}^{M}$. 
Solving the corresponding dual problem requires then the evaluation of all pairwise \emph{kernel} evaluations $\kappa(\mat{x}_{n},\mat{x}_{n'}) \isdef {\mat{\phi}(\mat{x}_{n})}\trans\,\mat{\phi}(\mat{x}_{n'})$. The so-called \emph{kernel trick} ensures that these kernel evaluations can be computed without ever explicitly mapping the inputs to the higher-dimensional space.

This approach has been applied to a plethora of methods \citep{scholkopf_learning_2002}, e.g. \emph{Least-Squares Support Vector Machine} (\gls*{ls-svm}) \citep{suykens_least_1999}, \emph{Ridge Regression} (\gls*{rr}) \citep{saunders_ridge_1998}, \emph{Lasso Regression} (\gls*{lr}) \citep{roth_generalized_2004}, \emph{Principal Component Analysis} (\gls*{pca}) \citep{mika_kernel_1998,suykens_support_2003} and more recently to the \emph{Singular Value Decomposition} (\gls*{svd}) \citet{suykens_svd_2016} in order to yield their \emph{kernelized} variants.

Among those, \gls*{pca} is an ubiquitous unsupervised learning approach which seeks an orthogonal subspace that maximizes the covariance between the samples of data.
Its kernelized counterpart, \gls*{kpca}, seeks an orthogonal subspace which maximizes the covariance of samples of data mapped into a higher dimensional space.
By construction, \gls*{kpca} does not provide any information regarding the row subspace of the data matrix, meaning that if there is any asymmetry in the data, it will not be captured.
A related but different method is the \gls*{svd}, which factors a data matrix in terms of orthogonal row and column subspaces linked by a positive diagonal matrix of so-called \emph{singular values}. The \gls*{svd} has been cast in the primal-dual framework by \citet{suykens_svd_2016}, who also proposed a \gls*{ksvd} extension in terms of feature maps of both rows and columns of the data matrix, coupled together by a \emph{compatibility matrix}. 
In contrast with \gls*{kpca}, said construction gives rise to kernel functions that can be asymmetric and non-positive and thus are arguably better suited to model real-life data, which often is nonlinear and asymmetric \citep{tao_nonlinear_2023}.
The \emph{Multilinear Singular Value Decomposition} (\gls*{mlsvd}) \citep{de_lathauwer_multilinear_2000}
extends the concept of the \gls*{svd} to higher-order arrays, also known as tensors. In simple terms, the \gls*{mlsvd} factors a data tensor in terms of orthogonal subspaces corresponding to each mode coupled by a \emph{core tensor}, which unlike in the \gls*{svd} case, does not need to be diagonal. The \gls*{mlsvd} has numerous applications in signal and image processing, computer vision, chemometrics, finance, human motion analysis, data mining, machine learning and deep learning \citet{kolda_tensor_2009,cichocki_tensor_2016,cichocki_tensor_2017,panagakis_tensor_2021}.

In this paper we extend the Lanczos decomposition theorem of matrices \citep{lanczos_linear_1958} to the tensor case. 
This allows us to derive a primal-dual formulation of the \gls*{mlsvd}, which recovers as special cases both \gls*{pca} and the \gls*{svd}. 
The newly established primal formulation can be used to attain computational gains in the large sample regime, and allows us to kernelize the \gls*{mlsvd} by means of feature maps. These feature maps define the construction of a kernel tensor in the dual as opposed to a kernel matrix in the case of \gls*{pca} and \gls*{svd}. 
Similarly to the \gls*{svd} case, the tensor kernel does not need to be symmetric or positive.
We discuss possible choices of kernel functions and applications, which range from signal analysis to deep learning.

The remainder of the paper is structured as follows. In \cref{sec:background} we provide the background related to the \gls*{mlsvd} as well as a theorem that will be useful to prove our main result in \cref{sec:proof}. We discuss related work in \cref{sec:related_work} and formulate recommendations for future work in \cref{sec:conclusion}.

\section{Background}\label{sec:background}
In the remainder of this paper we denote tensors with uppercase calligraphic bold e.g. $\ten{X}$, matrices in uppercase bold e.g. $\mat{X}$, vectors in lowecase bold $\mat{x}$ and scalars in lowercase, e.g. $x$. We denote the mode-$d$ unfolding of a tensor $\ten{X}$ \citep{kolda_tensor_2009} with $\mat{X}_{(d)}$, the Kronecker product with $\kron$ and column-major vectorization with $\vectorize{\cdot}$.

The compact \gls*{svd} of a rank-$R$ matrix $\mat{X} \in \mathbb{R}^{N_1 \times N_2}$ can be written as $\vectorize{\mat{X}} = \left( \mat{U}_2 \kron \mat{U}_1 \right)\vectorize{\mat{S}}$, with semi-orthogonal matrices $\mat{U}_1 \in \mathbb{R}^{N_1 \times R}$ and $\mat{U}_2 \in \mathbb{R}^{N_2 \times R}$ such that $\mat{U}_{1}\trans \mat{U}_{1} = \mat{I}_{R_1}$, $\mat{U}_{2}\trans \mat{U}_{2} = \mat{I}_{R_2}$ and a square diagonal matrix $\mat{S} \in \mathbb{R}^{R \times R}$ of singular values.
The \gls*{mlsvd}, also known as the \gls*{hosvd}, is one way to generalize the \gls*{svd} of matrices to higher-order tensors \citep{de_lathauwer_multilinear_2000}, which expresses a $D$th-order tensor in terms of $D$ coupled orthogonal subspaces.
In order to simplify notation and increase the readability we will consider from now on the case $D=3$ without any loss of generality. For the general case we refer the reader to \cref{appendixA}.
\begin{definition}[\emph{Multilinear Singular Value Decomposition} (\gls*{mlsvd}) \citep{de_lathauwer_multilinear_2000}]\label{def:mlsvd}
        The rank-$(R_1,R_2,R_3)$ \gls*{mlsvd} of a $3$rd-order tensor $\ten{X} \in \mathbb{R}^{N_1 \times N_2 \times N_3}$ is
        \begin{align}
            \vectorize{\ten{X}} &= \left(\mat{U}_3 \kron  \mat{U}_2 \kron \mat{U}_1 \right)\vectorize{\ten{S}},
            \label{eq:mlsvd}
        \end{align}
    where $\mat{U}_1 \in \mathbb{R}^{N_1 \times R_1} ,\mat{U}_2 \in \mathbb{R}^{N_2 \times R_2},\mat{U}_3 \in \mathbb{R}^{N_3 \times R_3}$ are semi-orthogonal factor matrices and $\ten{S} \in \mathbb{R}^{R_1 \times R_2 \times R_3}$ is the \emph{core} tensor such that the matrices $\mat{S}_{(1)}\mat{S}_{(1)}\trans$, $\mat{S}_{(2)}\mat{S}_{(2)}\trans$, $\mat{S}_{(3)}\mat{S}_{(3)}\trans$ are positive diagonal.
\end{definition}
In practice, the \gls*{mlsvd} of a $3$-rd order tensor can be computed by $3$ \gls*{svd} factorizations as elucidated in \citet{de_lathauwer_multilinear_2000}. Each \gls*{svd} provides $\mat{U}_1$, $\mat{U}_2$ and $\mat{U}_3$ as the left-singular values of the respective unfoldings $\mat{X}_{(1)}$, $\mat{X}_{(2)}$, $\mat{X}_{(3)}$. The core tensor is then computed by solving \cref{eq:mlsvd} for $\ten{S}$.

\section{A Primal-Dual formulation for the \gls*{mlsvd}}\label{sec:proof}
Before presenting our main result, i.e. a primal-dual formulation of the \gls*{mlsvd}, we need to generalize an important theorem by \citet{lanczos_linear_1958}, who defines \emph{shifted} eigenvalue problems which are equivalent to the \gls*{svd}.
Similarly, the \gls*{mlsvd} of a $3$rd-order tensor $\ten{X} \in \mathbb{R}^{N_1 \times N_2 \times N_3}$ can then be uniquely defined as a set of $3$ coupled matrix equations. A generalization of the proof to tensors of higher order is straightforward.
\begin{theorem}[Generalized Lanczos decomposition theorem]\label{lemma:mlsvd}
     An arbitrary rank-$(R_1,R_2,R_3)$ tensor $\ten{X} \in \mathbb{R}^{N_1 \times N_2 \times N_3}$ can be written in \gls*{mlsvd} form, i.e. as in \cref{eq:mlsvd} with core tensor $\ten{S}\in\mathbb{R}^{R_1\times R_2 \times R_3}$ and semi-orthogonal factor matrices $\mat{U}_{(1)}\in\mathbb{R}^{N_1\times R_1}$, $\mat{U}_{(2)}\in\mathbb{R}^{N_2\times R_2}$ and $\mat{U}_{(3)}\in\mathbb{R}^{N_3\times R_3}$ defined by the following set of equations 
    \begin{align}
    \begin{split}\label{eq:lanczos}
    \mat{U}_1\, \mat{S}_{(1)} &= \mat{X}_{(1)} \left(\mat{U}_3 \kron \mat{U}_2 \right),\\
    \mat{U}_2\, \mat{S}_{(2)} &= \mat{X}_{(2)} \left(\mat{U}_3 \kron \mat{U}_1 \right),\\
    \mat{U}_3\, \mat{S}_{(3)} &= \mat{X}_{(3)} \left(\mat{U}_2 \kron \mat{U}_1 \right),
    \end{split}
    \end{align}
    with the additional constraint that $\mat{S}_{(1)}\mat{S}_{(1)}\trans$,  $\mat{S}_{(2)}\mat{S}_{(2)}\trans$,  $\mat{S}_{(3)}\mat{S}_{(3)}\trans$ are positive diagonal matrices.
\end{theorem}
\begin{proof}
    The proof is divided in two steps, first we show that the factor matrices $\mat{U}_1$, $\mat{U}_2$, $\mat{U}_3$ are semi-orthogonal, second we show that indeed \cref{eq:lanczos} implies the \gls*{mlsvd} i.e. \cref{eq:mlsvd}.
     We begin by left-multiplying each side of \cref{eq:lanczos} respectively with $\mat{U}_1\trans$, $\mat{U}_2\trans$, $\mat{U}_3\trans$, resulting in three equations of the form
    \begin{align}
    \begin{split}\label{eq:proof_1}
        \mat{U}_1\trans \mat{U}_1 \mat{S}_{(1)} &= \mat{U}_1\trans \mat{X}_{(1)} \left(\mat{U}_3 \kron \mat{U}_2 \right) \defis \mat{D}_{(1)},\\
        \mat{U}_2\trans \mat{U}_2 \mat{S}_{(2)} &= \mat{U}_2\trans \mat{X}_{(2)} \left(\mat{U}_3 \kron \mat{U}_1 \right) \defis  \mat{D}_{(2)},\\
        \mat{U}_3\trans \mat{U}_3 \mat{S}_{(3)} &= \mat{U}_3\trans \mat{X}_{(3)} \left(\mat{U}_2 \kron \mat{U}_1 \right) \defis \mat{D}_{(3)}.
    \end{split}
    \end{align}
    By construction, all three right-hand sides of \cref{eq:proof_1} are different unfoldings of the same tensor $\ten{D}\in\mathbb{R}^{R_1\times R_2 \times R_3}$ for any choice of $\ten{X}$ and $\mat{U}_{1}$, $\mat{U}_{2}$, $\mat{U}_{3}$. 
    Vectorizing both sides of the equations yields
    \begin{align*}
        (\mat{I}_{R_3} \kron \mat{I}_{R_2} \kron \mat{U}_1\trans \mat{U}_1) \vectorize{\ten{S}} = \vectorize{\ten{D}},\\
        (\mat{I}_{R_3} \kron \mat{U}_2\trans \mat{U}_2 \kron \mat{I}_{R_1}) \vectorize{\ten{S}} = \vectorize{\ten{D}},\\
        (\mat{U}_3\trans \mat{U}_3 \kron \mat{I}_{R_2} \kron \mat{I}_{R_1}) \vectorize{\ten{S}} = \vectorize{\ten{D}},
    \end{align*}
    Equating any two out of the $3 \choose 2$ pairs of equations e.g. the first one with the second one results in
    \begin{align*}
        &(\mat{I}_{R_3} \kron \mat{I}_{R_2} \kron \mat{U}_1\trans \mat{U}_1) \vectorize{\ten{S}} \\
        =& (\mat{I}_{R_3} \kron \mat{U}_2\trans \mat{U}_2 \kron \mat{I}_{R_1}) \vectorize{\ten{S}}.
    \end{align*}
    This equality holds if $\mat{U}_1\trans \mat{U}_1=\mat{U}_2\trans \mat{U}_2 = \mat{0} $ (trivial solution), which we do not consider. If $\mat{U}_2\trans \mat{U}_2$ is full-rank and thus invertible, the right-hand side is invertible. Left-multiplying by the inverse of the right-hand side yields
    \begin{align*}
        &(\mat{I}_{R_3} \kron \mat{I}_{R_2} \kron \mat{I}_{R_1})\vectorize{\ten{S}} \\
        =&(\mat{I}_{R_3} \kron \mat{U}_2\trans \mat{U}_2 \kron \mat{I}_{R_1})^{-1} (\mat{I}_{R_3} \kron \mat{I}_{R_2} \kron \mat{U}_1\trans \mat{U}_1) \vectorize{\ten{S}}  \\
        =&(\mat{I}_{R_3} \kron (\mat{U}_2\trans \mat{U}_2)^{-1} \kron \mat{U}_1\trans \mat{U}_1) \vectorize{\ten{S}},
    \end{align*}
    where the second equality follows from the mixed-product property, see \citet{loan_ubiquitous_2000}.
    The equality holds if and only if $(\mat{U}_2\trans \mat{U}_2)^{-1} = \mat{I}_{R_2}$ and $\mat{U}_{1}\trans \mat{U}_{1} = \mat{I}_{R_1}$.
    Repeating the argument with at least $\lceil \frac{3}{2} \rceil$ unique pairs out of the $3 \choose 2$ pairs of equations yields, apart from the trivial $\mat{U}_1=\mat{U}_2=\mat{U}_3=\mat{0}$ solution,
        \begin{align}
        \begin{split}\label{eq:proof_2}
            \mat{U}_1\trans \mat{U}_1 = \mat{I}_{R_1},\\
            \mat{U}_2\trans \mat{U}_2 = \mat{I}_{R_2},\\
            \mat{U}_3\trans \mat{U}_3 = \mat{I}_{R_3},
        \end{split}
        \end{align}
    which implies that $\mat{U}_1$, $\mat{U}_2$ and $\mat{U}_3$ are semi-orthogonal.
    Right-multiplying both sides of \cref{eq:lanczos} by respectively $\mat{S}_{(1)}\trans \mat{U}_1\trans$, $\mat{S}_{(2)}\trans \mat{U}_2\trans$, $\mat{S}_{(3)}\trans \mat{U}_3\trans$ yields
    \begin{align}
    \begin{split}\label{eq:proof_3}
        \mat{U}_1\, \mat{S}_{(1)} \mat{S}_{(1)}\trans \mat{U}_1\trans &= \mat{X}_{(1)} \left(\mat{U}_3 \kron \mat{U}_2 \right) \mat{S}_{(1)}\trans \mat{U}_1\trans,\\
        \mat{U}_2\, \mat{S}_{(2)} \mat{S}_{(2)}\trans \mat{U}_2\trans &= \mat{X}_{(2)} \left(\mat{U}_3 \kron \mat{U}_1 \right) \mat{S}_{(2)}\trans \mat{U}_2\trans,\\
        \mat{U}_3\, \mat{S}_{(3)} \mat{S}_{(3)}\trans \mat{U}_3\trans &= \mat{X}_{(3)} \left(\mat{U}_2 \kron \mat{U}_1 \right) \mat{S}_{(3)}\trans \mat{U}_3\trans.
    \end{split}
    \end{align}    
    The left-hand side is the eigendecomposition of the right-hand side of \cref{eq:proof_3} due to the assumption that $\mat{S}_{(1)}\mat{S}_{(1)}\trans$, $\mat{S}_{(2)}\mat{S}_{(2)}\trans$, $\mat{S}_{(3)}\mat{S}_{(3)}\trans$ are positive diagonal matrices (eigenvalues) and the above proof that $\mat{U}_1$, $\mat{U}_2$, $\mat{U}_3$ are semi-orthogonal matrices (eigenvectors). From \cref{eq:proof_3} it follows that $\mat{U}_1$ is an orthogonal basis for the column space of $\mat{X}_{(1)}$, and likewise for the other unfoldings. We can therefore write
    \begin{align}
    \begin{split}\label{eq:proof_4}
        \mat{U}_1\, \mat{S}_{(1)} \mat{S}_{(1)}\trans \mat{U}_1\trans &= \mat{U}_{1} \mat{R}_1 \left(\mat{U}_3 \kron \mat{U}_2 \right) \mat{S}_{(1)}\trans \mat{U}_1\trans,\\
        \mat{U}_2\, \mat{S}_{(2)} \mat{S}_{(2)}\trans \mat{U}_2\trans &= \mat{U}_{2} \mat{R}_2 \left(\mat{U}_3 \kron \mat{U}_1 \right) \mat{S}_{(2)}\trans \mat{U}_2\trans,\\
        \mat{U}_3\, \mat{S}_{(3)} \mat{S}_{(3)}\trans \mat{U}_3\trans &= \mat{U}_{3} \mat{R}_3\left(\mat{U}_2 \kron \mat{U}_1 \right) \mat{S}_{(3)}\trans \mat{U}_3\trans,
    \end{split}
    \end{align}    
    where $\mat{R}_1\in\mathbb{R}^{R_1 \times N_2 N_3}$, $\mat{R}_2\in\mathbb{R}^{R_2\times N_1 N_3}$ and $\mat{R_3}\in\mathbb{R}^{R_3 \times N_1 N_2}$ are general coefficient matrices.
    From \cref{eq:proof_4} follows that
    \begin{align}
    \begin{split}\label{eq:proof_5}
        \mat{S}_{(1)} = \mat{R}_1 \left(\mat{U}_3 \kron \mat{U}_2 \right),\\
        \mat{S}_{(2)} = \mat{R}_2 \left(\mat{U}_3 \kron \mat{U}_1 \right),\\
        \mat{S}_{(3)} = \mat{R}_3 \left(\mat{U}_2 \kron \mat{U}_1 \right),
    \end{split}
    \end{align}  
    which by the semi-orthogonality of $\mat{U}_1$, $\mat{U}_2$ and $\mat{U}_3$ is satisfied if and only if
    \begin{align}
    \begin{split}\label{eq:proof_6}
         \mat{R}_1 = \mat{S}_{(1)}\left(\mat{U}_3\trans \kron \mat{U}_2\trans \right) ,\\
         \mat{R}_2 = \mat{S}_{(2)}\left(\mat{U}_3\trans \kron \mat{U}_1\trans \right) ,\\
         \mat{R}_3 = \mat{S}_{(3)}\left(\mat{U}_2\trans \kron \mat{U}_1\trans \right) .\\
    \end{split}
    \end{align}  
    Substitution of \cref{eq:proof_6} into $\mat{X}_{(1)}=\mat{U}_1\,\mat{R}_1$, $\mat{X}_{(2)}=\mat{U}_2\,\mat{R}_2$, $\mat{X}_{(3)}=\mat{U}_3\,\mat{R}_3$ we conclude that $\vectorize{\ten{X}} = \left(\mat{U}_3 \kron  \mat{U}_2 \kron \mat{U}_1 \right)\vectorize{\ten{S}}$, which is the defining \cref{eq:mlsvd} of the \gls*{mlsvd} as in \cref{def:mlsvd}.
\end{proof}
Equipped with \cref{lemma:mlsvd}, we can now formulate the \gls*{mlsvd} as a primal-dual optimization problem by defining the \emph{primal} \gls*{mlsvd} optimization problem in its most general context.
\begin{definition}[Primal \gls*{mlsvd} optimization problem]\label{def:primal_problem}
Given three feature matrices $\mat{\Phi}_1 \in \mathbb{R}^{N_1 \times M_1},\mat{\Phi}_2\in\mathbb{R}^{N_2 \times M_2},\mat{\Phi}_3\in\mathbb{R}^{N_3 \times M_3}$, a compatibility tensor $\ten{C}\in\mathbb{R}^{M_1\times M_2\times M_3}$ and regularization parameters $\ten{S}\in\mathbb{R}^{R_1\times R_2\times R_3}$ such that $\mat{S}_{(1)}\mat{S}_{(1)}\trans$, $\mat{S}_{(2)}\mat{S}_{(2)}\trans$ and $\mat{S}_{(3)}\mat{S}_{(3)}\trans$ are positive diagonal matrices, we define the primal optimization problem as
    \begin{align}
    \begin{split}
        \label{eq:primal_problem}
        &\max_{\mat{W}_1,\mat{W}_2,\mat{W}_3,\mat{E}_1,\mat{E}_2,\mat{E}_3} J(\mat{W}_1,\mat{W}_2,\mat{W}_3,\mat{E}_1,\mat{E}_2,\mat{E}_3) \isdef \\
                &  \frac{1}{2}\sum_{d=1}^3 \Tr{\left(\mat{E}_d \, (\mat{S}_{(d)}\mat{S}_{(d)}\trans)^{-1}\, \mat{E}_d\trans \right)} \\
                &- 2\,\vectorize{\ten{C}}\trans\,\left(\mat{W}_3 \kron \mat{W}_2 \kron \mat{W}_1 \right)\, \vectorize{\ten{S}} \\
                & + \frac{1}{2} \vectorize{\ten{C}}^T \left(\mat{\Phi}_3\trans \mat{\Phi}_3 \kron \mat{\Phi}_2\trans \mat{\Phi}_2  \kron \mat{\Phi}_1\trans \mat{\Phi}_1\right) \vectorize{\ten{C}} \\
                \textrm{s.t.:} \ &  \mat{E}_1 = \mat{\Phi}_1 \, \mat{C}_{(1)} \, \left( \mat{W}_3 \kron \mat{W}_2 \right) \, \mat{S}_{(1)}\trans,\\
    & \mat{E}_2= \mat{\Phi}_2 \, \mat{C}_{(2)} \, \left( \mat{W}_3 \kron \mat{W}_1 \right) \, \mat{S}_{(2)}\trans,\\
    & \mat{E}_3 =  \mat{\Phi}_3 \, \mat{C}_{(3)} \, \left( \mat{W}_2 \kron \mat{W}_1 \right) \, \mat{S}_{(3)}\trans.
    \end{split}
    \end{align}
\end{definition}
In \cref{eq:primal_problem} we seek weights matrices $\mat{W}_1\in\mathbb{R}^{M_1\times R_1}$, $\mat{W}_2\in \mathbb{R}^{M_2\times R_2}$, $\mat{W}_3 \in \mathbb{R}^{M_3\times R_3}$ and error matrices $\mat{E}_1 \in \mathbb{R}^{N_1 \times R_1}$, $\mat{E}_2\in\mathbb{R}^{N_2 \times R_2}$, $\mat{E}_3\in\mathbb{R}^{N_3\times R_3}$ that maximize an objective function $J$ composed of a term that maximizes the variance associated with each feature matrix, a regularization term that acts on the weights and an optional constant term that ensures that the cost at the optimum is zero. 
For now we assume that the features $\mat{\Phi}$, compatibility tensor $\ten{C}$ and the regularization parameter $\ten{S}$ are given and not necessarily data-dependent, we will later see examine relevant choices.
We now establish a link between the \emph{primal} \gls*{mlsvd} optimization problem presented in \cref{def:primal_problem} and the \emph{dual} \gls*{mlsvd} optimization problem.
\begin{theorem}\label{thm:mlsvd}
    The dual optimization problem associated with the primal optimization problem of \cref{def:primal_problem} is the \gls*{mlsvd} of the kernel tensor $\ten{K}\in\mathbb{C}^{N_1\times N_2\times N_3}$. The kernel tensor $\ten{K}$ is defined as
    \begin{equation}
    \label{eq:Kten}
        \vectorize{\ten{K}} \isdef \left(\mat{\Phi}_3 \kron \mat{\Phi}_2 \kron \mat{\Phi}_1 \right) \; \vectorize{\ten{C}}.
    \end{equation}
    \label{thm:nonlinear_mlsvd}
\end{theorem}
\begin{proof}
Consider the Lagrangian $\mathcal{L}$ associated with the primal optimization problem of \cref{def:primal_problem},
\begin{align*}
    &\mathcal{L}(\mat{W}_1,\mat{W}_2,\mat{W}_3,\mat{E}_1,\mat{E}_2,\mat{E}_3,\mat{U}_1,\mat{U}_2,\mat{U}_3) \\
    &= J(\mat{W}_1,\mat{W}_2,\mat{W}_3,\mat{E}_1,\mat{E}_2,\mat{E}_3) \\  
    &-\Tr{\left( \left(\mat{E}_1-\mat{\Phi}_1 \, \mat{C}_{(1)} \, \left( \mat{W}_3 \kron \mat{W}_2 \right) \, \mat{S}_{(1)}\trans \right) \mat{U}_1\trans \right)}\\
    &-\Tr{\left( \left(\mat{E}_2-\mat{\Phi}_2 \, \mat{C}_{(2)} \, \left( \mat{W}_3 \kron \mat{W}_1 \right) \, \mat{S}_{(2)}\trans \right) \mat{U}_2\trans \right)}\\
    &-\Tr{\left( \left(\mat{E}_3-\mat{\Phi}_3 \, \mat{C}_{(3)} \, \left( \mat{W}_2 \kron \mat{W}_1 \right) \, \mat{S}_{(3)}\trans \right) \mat{U}_3\trans \right)},
\end{align*}
where $\mat{U}_1\in\mathbb{R}^{N_1 \times R_1}$, $\mat{U}_2\in\mathbb{R}^{N_2 \times R_2}$, $\mat{U}_3\in\mathbb{R}^{N_3 \times R_3}$ are Lagrange multiplier matrices.
The \gls*{kkt} conditions result in
\begin{align*}
    &\frac{\partial \mathcal{L}}{\partial \mat{W}_1} = \mat{0} \iff 2\, {\mat{C}_{(1)}} (\mat{W}_3 \kron \mat{W}_2) \, \mat{S}_{(1)}\trans = {\mat{C}_{(1)}} \left( \mat{W}_3 \kron \mat{\Phi}_2\trans \mat{U}_2  +  \mat{\Phi}_3\trans \mat{U}_3 \kron \mat{W}_2 \right) \mat{S}_{(1)}\trans, \\
    &\frac{\partial \mathcal{L}}{\partial \mat{W}_2} = \mat{0} \iff  2\, {\mat{C}_{(2)}} (\mat{W}_3 \kron \mat{W}_1) \, \mat{S}_{(2)}\trans = {\mat{C}_{(2)}} \left( \mat{W}_3 \kron \mat{\Phi}_1\trans \mat{U}_1  +  \mat{\Phi}_3\trans \mat{U}_3 \kron \mat{W}_1 \right) \mat{S}_{(2)}\trans, \\
    &\frac{\partial \mathcal{L}}{\partial \mat{W}_3} = \mat{0} \iff  2\, {\mat{C}_{(3)}} (\mat{W}_2 \kron \mat{W}_1) \, \mat{S}_{(3)}\trans = {\mat{C}_{(3)}} \left( \mat{W}_2 \kron \mat{\Phi}_1\trans \mat{U}_1  +  \mat{\Phi}_2\trans \mat{U}_2 \kron \mat{W}_1 \right) \mat{S}_{(3)}\trans, \\  
    &\frac{\partial \mathcal{L}}{\partial \mat{E}_1} = \mat{0} \iff \mat{E}_1 = \mat{U}_1 \mat{S}_{(1)} \mat{S}_{(1)}\trans ,  \\
    &\frac{\partial \mathcal{L}}{\partial \mat{E}_2} = \mat{0} \iff \mat{E}_2 = \mat{U}_2 \mat{S}_{(2)} \mat{S}_{(2)}\trans ,  \\
    &\frac{\partial \mathcal{L}}{\partial \mat{E}_3} = \mat{0} \iff \mat{E}_3 = \mat{U}_3 \mat{S}_{(3)} \mat{S}_{(3)}\trans , \\
    &\frac{\partial \mathcal{L}}{\partial \mat{U}_1} = \mat{0} \iff  \mat{E}_1 = \mat{\Phi}_1 \, \mat{C}_{(1)} \, \left( \mat{W}_3 \kron \mat{W}_2 \right) \, \mat{S}_{(1)}\trans,\\
    &\frac{\partial \mathcal{L}}{\partial \mat{U}_2} = \mat{0} \iff \mat{E}_2= \mat{\Phi}_2 \, \mat{C}_{(2)} \, \left( \mat{W}_3 \kron \mat{W}_1 \right) \, \mat{S}_{(2)}\trans,\\
    &\frac{\partial \mathcal{L}}{\partial \mat{U}_3} = \mat{0} \iff \mat{E}_3= \mat{\Phi}_3 \, \mat{C}_{(3)} \, \left( \mat{W}_2 \kron \mat{W}_1 \right) \, \mat{S}_{(3)}\trans.
\end{align*}
The equality of the first three \gls*{kkt} conditions holds for the trivial solution $\mat{W}_1=\mat{0},\mat{W}_2=\mat{0},\mat{W}_3=\mat{0}$ or when
\begin{align*}
    \mat{\Phi}_1\trans \mat{U}_1 = \mat{W}_1, \; \mat{\Phi}_2\trans \mat{U}_2 = \mat{W}_2,  \; \mat{\Phi}_3\trans \mat{U}_3 =  \mat{W}_3.
\end{align*}
These expressions can be used to eliminate $\mat{W}_1$ and $\mat{E}_1$ to obtain
\begin{align*}
    \mat{U}_1 \mat{S}_{(1)} \mat{S}_{(1)}\trans &=  \mat{\Phi}_1 \, \mat{C}_{(1)} \, \left( \mat{\Phi}_3\trans  \kron \mat{\Phi}_2\trans \right) \left(\mat{U}_3 \kron \mat{U}_2 \right) \, \mat{S}_{(1)}\trans\\
     &= \mat{K}_{(1)} \left(\mat{U}_3 \kron \mat{U}_2 \right) \, \mat{S}_{(1)}\trans.
\end{align*}
The definition of $\ten{K}$ in \cref{eq:Kten} was used to write the second equality. Similarly we obtain expressions for the Lagrange multipliers $\mat{U}_2$ and $\mat{U}_3$ leading to
\begin{align*}
    \mat{U}_1 \mat{S}_{(1)} & = \mat{K}_{(1)} \left(\mat{U}_3 \kron \mat{U}_2 \right),\\
    \mat{U}_2 \mat{S}_{(2)} & = \mat{K}_{(2)} \left(\mat{U}_3 \kron \mat{U}_1 \right),\\
    \mat{U}_3 \mat{S}_{(3)} & = \mat{K}_{(3)} \left(\mat{U}_2 \kron \mat{U}_1 \right),
\end{align*}
which by \cref{lemma:mlsvd} is the \gls*{mlsvd} of the tensor $\ten{K}\in\mathbb{R}^{N_1\times N_2 \times N_3}$.
\end{proof}
\Cref{thm:mlsvd} establishes that instead of computing the \gls*{mlsvd} of a tensor $\ten{K}$, one can alternatively solve the optimization problem in \cref{def:primal_problem}. We now prove that the objective function $J$ is equal to zero in the \gls*{mlsvd} solution.
\begin{corollary}
The \gls*{mlsvd} solution of the dual problem in \cref{thm:mlsvd} result in a zero objective function $(J=0)$ in the primal optimization problem of \cref{def:primal_problem}.
\end{corollary}
\begin{proof}
Plugging in the \gls*{kkt} conditions for $\mat{E}_1,\mat{E}_2,\mat{E}_3$ and $\mat{W}_1,\mat{W}_2,\mat{W}_3$ into the primal objective function $J$ results in
\begin{align*}
\phantom{-}&\frac{1}{2}\sum_{d=1}^3 \Tr{\left(\mat{E}_d \, (\mat{S}_{(d)}\mat{S}_{(d)}\trans)^{-1}\, \mat{E}_d\trans \right)} - 2\vectorize{\ten{C}}\trans\,\left(\mat{W}_3 \kron \mat{W}_2 \kron \mat{W}_1 \right)\, \vectorize{\ten{S}} 
+ \frac{1}{2} \vectorize{\ten{K}}^T \vectorize{\ten{K}},\\
=&\frac{1}{2}\sum_{d=1}^3 \Tr{\left(\mat{U}_d \, \mat{S}_{(d)}\mat{S}_{(d)}\trans\, \mat{U}_d\trans \right)} 
- 2\vectorize{\ten{C}}\trans\,\left(\mat{\Phi}_3\trans \mat{U}_3 \kron \mat{\Phi}_2\trans \mat{U}_2 \kron \mat{\Phi}_1\trans \mat{U}_1 \right)\, \vectorize{\ten{S}} + \frac{1}{2} \vectorize{\ten{K}}^T \vectorize{\ten{K}}   \\
=&\left( \frac{3}{2} - 2+ \frac{1}{2}\right) \vectorize{\ten{K}}^T \vectorize{\ten{K}} =0. \\
\end{align*}
The first three variance terms are simplified using the cyclic permutation invariance of the Frobenius trace norm. The regularization term is simplified using the definition of the kernel tensor $\ten{K}$ and its \gls*{mlsvd}. The third term in the objective function is also simplified using the definition of the kernel tensor $\ten{K}$.
\end{proof}

\begin{figure*}
    \centering
    \begin{subcaptionblock}[T][][c]{.48\textwidth}
    \centering
    \begin{tikzpicture}
    \node (first) at (0, 0) {
    \begin{tikzpicture}
    \SetVertexStyle[MinSize=0.65\DefaultUnit]
    \Vertex[x=1,y=0,shape=circle,label=$\ten{C}$,fontsize=\normalsize,RGB,color={255,255,255}]{c}
    \Vertex[x=1,y=-3,shape=circle,label=$\ten{S}$,fontsize=\normalsize,RGB,color={255,255,255}]{s}
    \Vertex[x=0,y=-1,shape=circle,label=$\mat{\Phi}_1$,fontsize=\normalsize,RGB,color={255,255,255}]{phi}
    \Vertex[x=1,y=-1,shape=circle,label=$\mat{W}_2$,fontsize=\normalsize,RGB,color={255,255,255}]{w2}
    \Vertex[x=2,y=-1,shape=circle,label=$\mat{W}_3$,fontsize=\normalsize,RGB,color={255,255,255}]{w3}
    \Vertex[x=0,y=-2,shape=circle,fontsize=\normalsize,RGB,color={255,255,255},Pseudo]{e}
    \Edge[label=$M_1$,position=right](c)(phi)
    \Edge[label=$M_2$,position=right](c)(w2)
    \Edge[label=$M_3$,position=right](c)(w3)
    \Edge[label=$R_1$,position=right](s)(e)
    \Edge[label=$R_2$,position=right](s)(w2)
    \Edge[label=$R_3$,position=right](s)(w3)
    \Edge[label=$N_1$,position=right](phi)(e)
\end{tikzpicture}
    };
    \node (second) at (3, 0) {
    \begin{tikzpicture}
    \SetVertexStyle[MinSize=0.65\DefaultUnit]
    \Vertex[x=1,y=0,shape=circle,label=$\ten{C}$,fontsize=\normalsize,RGB,color={255,255,255}]{c}
    \Vertex[x=1,y=-3,shape=circle,label=$\ten{S}$,fontsize=\normalsize,RGB,color={255,255,255}]{s}
    \Vertex[x=1,y=-1,shape=circle,label=$\mat{\Phi}_2$,fontsize=\normalsize,RGB,color={255,255,255}]{phi}
    \Vertex[x=0,y=-1.5,shape=circle,label=$\mat{W}_1$,fontsize=\normalsize,RGB,color={255,255,255}]{w1}
    \Vertex[x=2,y=-1.5,shape=circle,label=$\mat{W}_3$,fontsize=\normalsize,RGB,color={255,255,255}]{w3}
    \Vertex[x=1,y=-2,shape=circle,fontsize=\normalsize,RGB,color={255,255,255},Pseudo]{e}
    \Edge[label=$M_2$,position=right](c)(phi)
    \Edge[label=$M_1$,position=right](c)(w1)
    \Edge[label=$M_3$,position=right](c)(w3)
    \Edge[label=$R_2$,position=right](s)(e)
    \Edge[label=$R_1$,position=right](s)(w1)
    \Edge[label=$R_3$,position=right](s)(w3)
    \Edge[label=$N_2$,position=right](phi)(e)
\end{tikzpicture}
    };
    \node (third) at (6, 0) {
    \begin{tikzpicture}
    \SetVertexStyle[MinSize=0.65\DefaultUnit]
    \Vertex[x=1,y=0,shape=circle,label=$\ten{C}$,fontsize=\normalsize,RGB,color={255,255,255}]{c}
    \Vertex[x=1,y=-3,shape=circle,label=$\ten{S}$,fontsize=\normalsize,RGB,color={255,255,255}]{s}
    \Vertex[x=2,y=-1,shape=circle,label=$\mat{\Phi}_3$,fontsize=\normalsize,RGB,color={255,255,255}]{phi}
    \Vertex[x=0,y=-1.5,shape=circle,label=$\mat{W}_1$,fontsize=\normalsize,RGB,color={255,255,255}]{w1}
    \Vertex[x=1,y=-1.5,shape=circle,label=$\mat{W}_2$,fontsize=\normalsize,RGB,color={255,255,255}]{w2}
    \Vertex[x=2,y=-2,shape=circle,fontsize=\normalsize,RGB,color={255,255,255},Pseudo]{e}
    \Edge[label=$M_3$,position=right](c)(phi)
    \Edge[label=$M_1$,position=right](c)(w1)
    \Edge[label=$M_2$,position=right](c)(w2)
    \Edge[label=$R_3$,position=right](s)(e)
    \Edge[label=$R_1$,position=right](s)(w1)
    \Edge[label=$R_2$,position=right](s)(w2)
    \Edge[label=$N_3$,position=right](phi)(e)
\end{tikzpicture}
    };
\end{tikzpicture}
    
    \caption{Primal formulation (\ref{eq:primal}), from left to right of $\mat{E}_1$, $\mat{E}_2$ and $\mat{E}_3$.}
    \label{fig:primal}
    \end{subcaptionblock}%
    \hfill
    \begin{subcaptionblock}[T][][c]{.48\textwidth}
    \centering
    \begin{tikzpicture}
    \node (first) at (0, 0) {
    \begin{tikzpicture}
    \SetVertexStyle[MinSize=0.65\DefaultUnit]
    \Vertex[x=1,y=0,shape=circle,label=$\ten{K}$,fontsize=\normalsize,RGB,color={255,255,255}]{c}
    \Vertex[x=1,y=-3,shape=circle,label=$\ten{S}$,fontsize=\normalsize,RGB,color={255,255,255}]{s}
    \Vertex[x=1,y=-1.5,shape=circle,label=$\mat{U}_2$,fontsize=\normalsize,RGB,color={255,255,255}]{w2}
    \Vertex[x=2,y=-1.5,shape=circle,label=$\mat{U}_3$,fontsize=\normalsize,RGB,color={255,255,255}]{w3}
    \Vertex[x=0.0,y=-1.5,shape=circle,fontsize=\normalsize,RGB,color={255,255,255},Pseudo]{e}
    \Edge[label=$N_2$,position=right](c)(w2)
    \Edge[label=$N_3$,position=right](c)(w3)
    \Edge[label=$R_3$,position=right](s)(w3)
    \Edge[label=$R_1$,position=right](s)(w1)
    \Edge[label=$R_2$,position=right](s)(w2)
    \Edge[label=$N_1$,position=right](c)(e)
\end{tikzpicture}
    };
    \node (second) at (3, 0) {
    \begin{tikzpicture}
    \SetVertexStyle[MinSize=0.65\DefaultUnit]
    \Vertex[x=1,y=0,shape=circle,label=$\ten{K}$,fontsize=\normalsize,RGB,color={255,255,255}]{c}
    \Vertex[x=1,y=-3,shape=circle,label=$\ten{S}$,fontsize=\normalsize,RGB,color={255,255,255}]{s}
    \Vertex[x=0,y=-1.5,shape=circle,label=$\mat{U}_1$,fontsize=\normalsize,RGB,color={255,255,255}]{w1}
    \Vertex[x=2,y=-1.5,shape=circle,label=$\mat{U}_3$,fontsize=\normalsize,RGB,color={255,255,255}]{w3}
    \Vertex[x=1,y=-1.5,shape=circle,fontsize=\normalsize,RGB,color={255,255,255},Pseudo]{e}
    \Edge[label=$N_1$,position=right](c)(w1)
    \Edge[label=$N_3$,position=right](c)(w3)
    \Edge[label=$R_2$,position=right](s)(e)
    \Edge[label=$R_1$,position=right](s)(w1)
    \Edge[label=$R_3$,position=right](s)(w3)
    \Edge[label=$N_2$,position=right](c)(e)
\end{tikzpicture}
    };
    \node (third) at (6, 0) {
    \begin{tikzpicture}
    \SetVertexStyle[MinSize=0.65\DefaultUnit]
    \Vertex[x=1,y=0,shape=circle,label=$\ten{K}$,fontsize=\normalsize,RGB,color={255,255,255}]{c}
    \Vertex[x=1,y=-3,shape=circle,label=$\ten{S}$,fontsize=\normalsize,RGB,color={255,255,255}]{s}
    \Vertex[x=0,y=-1.5,shape=circle,label=$\mat{U}_1$,fontsize=\normalsize,RGB,color={255,255,255}]{w1}
    \Vertex[x=1,y=-1.5,shape=circle,label=$\mat{U}_2$,fontsize=\normalsize,RGB,color={255,255,255}]{w2}
    \Vertex[x=2,y=-1.5,shape=circle,fontsize=\normalsize,RGB,color={255,255,255},Pseudo]{e}
    \Edge[label=$N_1$,position=right](c)(w1)
    \Edge[label=$N_2$,position=right](c)(w2)
    \Edge[label=$R_3$,position=right](s)(e)
    \Edge[label=$R_1$,position=right](s)(w1)
    \Edge[label=$R_2$,position=right](s)(w2)
    \Edge[label=$N_3$,position=right](c)(e)
\end{tikzpicture}
    };
\end{tikzpicture}
    \caption{Dual formulation (\ref{eq:dual}), from left to right of $\mat{E}_1$, $\mat{E}_2$ and $\mat{E}_3$.}
    \label{fig:dual}
    \end{subcaptionblock}%
    \caption{Primal formulation (\cref{fig:primal}) and dual formulation (\cref{fig:dual}) in tensor network diagram notation. In these diagrams, each circle represent a tensor and each edge departing from a circle represents an index of the corresponding tensor. A connecting edge denotes then a summation along the corresponding index, an unconnected edge denotes a free index, see \citet{cichocki_tensor_2016} for a more in-depth explanation.}
    \label{fig:primal_dual}
\end{figure*}

The primal optimization problem of \cref{def:primal_problem} defines explicitly a model-based approach in terms of primal weights, which is equivalent to the \gls*{mlsvd} but operates in the vector space $\mathbb{R}^{M_1\times M_2\times M_3}$ instead of in the usual vector space $\mathbb{R}^{N_1\times N_2 \times N_3}$. Just like other learning problems that admit a primal-dual formulation, each representation has its own advantages in terms of computational complexity.
\begin{remark}[Primal and dual model representation]
    The \gls*{mlsvd} is characterized by a primal (\ref{eq:primal}) representation in terms of weights $\mat{W}_1$, $\mat{W}_2$, $\mat{W}_3$ and feature maps $\mat{\Phi}_1$, $\mat{\Phi}_2$, $\mat{\Phi}_3$ and a dual (\ref{eq:dual}) representation in terms of a kernel tensor $\ten{K}$ defined in \cref{eq:Kten} and Lagrange multipliers $\mat{U}_1$, $\mat{U}_2$, $\mat{U}_3$:
    \begin{align*}
        \mat{E}_1 &= \mat{\Phi}_1\, \mat{C}_{(1)} \,(\mat{W}_3 \kron \mat{W}_2) \, \mat{S}_{(1)}\trans, \\
        \tag{P} \label{eq:primal} \mat{E}_2 &= \mat{\Phi}_2\, \mat{C}_{(2)} \,(\mat{W}_3 \kron \mat{W}_1) \, \mat{S}_{(2)}\trans, \\
        \mat{E}_3 &= \mat{\Phi}_3\, \mat{C}_{(3)} \,(\mat{W}_2 \kron \mat{W}_1) \, \mat{S}_{(3)}\trans, \\
        \mat{E}_1 &= \mat{K}_{(1)} \,(\mat{U}_3 \kron \mat{U}_2) \, \mat{S}_{(1)}\trans, \\
        \tag{D}\label{eq:dual} \mat{E}_2 &=  \mat{K}_{(2)} \,(\mat{U}_3 \kron \mat{U}_1) \, \mat{S}_{(2)}\trans, \\
        \mat{E}_3 &=  \mat{K}_{(3)} \,(\mat{U}_2 \kron \mat{U}_1) \, \mat{S}_{(3)}\trans.
    \end{align*}
\end{remark}
An alternative representation of \cref{eq:primal,eq:dual} in tensor networks diagram notation is presented in \cref{fig:primal_dual}.
After precomputing $\mat{\Phi}_1\, \mat{C}_{(1)}$, $\mat{\Phi}_2\, \mat{C}_{(2)}$ and $\mat{\Phi}_3\, \mat{C}_{(3)}$, the primal formulation in \cref{eq:primal} is more convenient in terms of storage and computation when the number of \emph{samples} is larger than the feature space one operates in, i.e. $N\ll M$, requiring a computational and storage complexity of $\mathcal{O}(NM^2)$ with $N\isdef \max(N_1,N_2,N_3)$, $M\isdef \max(M_1,M_2,M_3)$. Alternatively, when the feature space is larger than the number of samples, the dual formulation in \cref{eq:dual} is more attractive, requiring a computational and storage complexity of $\mathcal{O}(N^3)$.
In particular, the dual formulation allows to operate implicitly in an infinite-dimensional feature space, as long as the kernel tensor $\ten{K}$ can be computed in closed form.
We will now examine possible choices of features and compatibility tensors $\ten{C}$, in particular such that the output of \cref{def:primal_problem} is equivalent to the \gls*{mlsvd} of a known data tensor $\ten{X}$, as well as nonlinear extensions and the relationships with other methods.

\subsection{Linear \gls*{mlsvd} of a Data Tensor}\label{sec:linear_mlsvd}
\Cref{thm:mlsvd} requires that the data that is fed into the primal optimization problem comes in the form of matrices. Its dual optimization problem is then the \gls*{mlsvd} of $\ten{K}$ in \cref{eq:Kten}. We will now see how the \gls*{mlsvd} of a general data tensor $\ten{X}$ can be obtained with suitable choices of features and compatibility tensor.
The \emph{linear} \gls*{mlsvd} of a general data tensor $\ten{X}\in\mathbb{R}^{N_1\times N_2 \times N_3}$ follows from \cref{thm:mlsvd} if one considers as features the mode-$d$ unfoldings of said tensor.
\begin{theorem}[Linear \gls*{mlsvd} of a data tensor]\label{thm:linear_mlsvd}
    Consider the data tensor $\ten{X}\in\mathbb{R}^{N_1\times N_2 \times N_3}$ with linear feature maps    
    \begin{equation}\label{eq:linear_features}
        \mat{\Phi}_{d} = \mat{X}_{(d)},
    \end{equation}
    which correspond to each unfolding of $\ten{X}$.
    If the compatibility tensor $\ten{C}\in\mathbb{C}^{N_2 N_3\times N_1 N_3\times N_1 N_2}$ satisfies 
    \begin{equation}\label{eq:compatibility}
        \vectorize{\ten{X}} = \left(\mat{X}_{(3)} \kron \mat{X}_{(2)}  \kron \mat{X}_{(1)}\right) \vectorize{\ten{C}},
    \end{equation}
    then solving the primal problem in \cref{def:primal_problem} yields the \gls*{mlsvd} of data tensor $\ten{X}\in\mathbb{C}^{N_1\times N_2\times N_3}$.
\end{theorem}
\begin{proof}
    The compatibility condition of \cref{eq:compatibility} associated with the linear features of \cref{eq:linear_features} satisfies the assumptions of \cref{thm:mlsvd}, yielding thus the \gls*{mlsvd} of data tensor $\ten{X}$.
\end{proof}
The linear compatibility condition of \cref{eq:compatibility} can always be satisfied when the dimensionalty of the data tensor $D \geq 3$, as the right-hand-side is then underdetermined and thus yields infinitely many solutions. A standard choice is then to choose the $\ten{C}$ with minimal Frobenius norm.
In the $D=2$ case, the compatibility condition reduces to $\vectorize{\mat{X}} = (\mat{X}\trans \kron \mat{X})\vectorize{\mat{C}}$, recovering the compatibility condition of the primal-dual formulation of the \gls*{svd} identified by \citet{suykens_svd_2016}. Said condition is however satisfied exactly only if $N_1=N_2$ yielding thus $\mat{C}=\mat{X}\inv$. When the matrix $\mat{X}$ is not square there is thus loss of information, i.e. $\mat{C}=\mat{X}\pinv$ if $N_1>N_2$ and $\mat{C}\trans = {\mat{X}\trans}\pinv$ if $N_1<N_2$, which is with high likelihood not the case for the \gls*{mlsvd}.

\subsection{\gls*{kmlsvd}}\label{sub:kernel_mlsvd}

The \gls*{kmlsvd} of a data tensor $\ten{X}\in\mathbb{R}^{N_1\times N_2 \times N_3}$ is defined by means of a set of general \emph{nonlinear} feature maps 
\begin{equation}\label{eq:nonlinear_features}
    \mat{\Phi}_{d} = \mat{\varPhi}_d({\mat{X}_{d}}) \in \mathbb{R}^{N_d\times M_d},
\end{equation}
which map each dataset or unfolding to a higher-dimensional nonlinear space $\mathbb{R}^{M_d}$.
Solving the primal optimization problem of \cref{def:primal_problem} is then equivalent to the \gls*{mlsvd} of the \emph{kernel tensor} $\ten{K}\in\mathbb{C}^{N_1\times N_2 \times N_3}$ in \cref{thm:mlsvd}, whose defining equation we provide once more:
\begin{equation*}
    \vectorize{\ten{K}} \isdef \left(\mat{\Phi}_3 \kron \mat{\Phi}_2 \kron \mat{\Phi}_1 \right) \; \vectorize{\ten{C}} .
\end{equation*}
The compatibility tensor $\ten{C}$ determines whether the kernel tensor function (and tensor kernel) are subject to any kind of symmetry or permutational invariances and together with the choice of feature map, positivity. 
The tensor kernel does not need to be symmetric or positive-definite, meaning that the \gls*{mlsvd} can represent asymmetric relationships that arise between the orthogonal subspaces where the high-dimensional data lives in terms of the core tensor.

The explicit computation of the kernel tensor scales exponentially in the dimensionality of the original data tensor $\ten{X}$ when carried out explicitly. This limitation can instead be bypassed by means of the so-called \emph{kernel trick}, which carries out the computations implicitly in the features spaces.
Two examples of kernel functions whose inputs live in the same space are the tensor-variate generalizations of the polynomial and exponential kernels described first by \citet{salzo_solving_2018-2, salzo_generalized_2020} in the context of $L_p$-regularized learning problems. We report here their slightly adjusted definition.
\begin{example}[Polynomial and exponential kernel \citep{salzo_solving_2018-2, salzo_generalized_2020}]\label{example:poly_exp_kernels}
    The polynomial tensor kernel function of degree $p\geq1$ is defined as
    \begin{equation*}
        \vectorize{\ten{K}_{\text{polynomial}}^p} \isdef  \left((\mat{X}_3 \kron \mat{X}_2 \kron \mat{X}_1) \vectorize{\ten{I}}\right)^p
    \end{equation*}
    where $\mat{X}_1\in\mathbb{R}^{N \times M}$, $\mat{X}_2\in\mathbb{R}^{N \times M}$, $\mat{X}_3\in\mathbb{R}^{N \times M}$ and $\ten{I}\in\mathbb{C}^{M\times M \times M}$ is the $3$rd-order diagonal identity tensor. It is implicitly assumed that the inputs live in the same space and can hence be coupled by means of the identity tensor.
    The exponential tensor kernel is then defined as
    \begin{equation*}
        \ten{K}_{\text{exponential}} \isdef \exp(\ten{K}_{\text{polynomial}}^1).
    \end{equation*}
    When $p$ is odd then the kernels are not positive-definite. The exponential kernel is defined implicitly by an infinite-dimensional power series feature map \citep{salzo_solving_2018-2}. 
\end{example}
Following \citet{suykens_svd_2016}, it is straightforward to define the features and compatibility tensor which yields a kernel tensor of elementwise nonlinearities. 
\begin{example}[Elementwise nonlinear kernel]\label{def:elementwise}
    Elementwise nonlinear kernels of a data tensor $\ten{X}\in\mathbb{R}^{N_1\times N_2 \times N_3}$ are defined as
    \begin{equation*}
        \ten{K}_{\text{elementwise}} \isdef f(\ten{X}),
    \end{equation*}
    where $f(\cdot)$ is any elementwise nonlinear function. The features are the linear features of \cref{thm:linear_mlsvd} and the compatibility tensor $\ten{C}\in\mathbb{R}^{M_1 \times M_2 \times M_3}$ satisfies \cref{eq:compatibility}.
\end{example}


The primal optimization problem of \Cref{def:primal_problem} encompasses many decompositions. In what follows we provide a brief overview of some examples.

\subsection{Kernel Orthogonal \gls*{cpd}}
The kernel \gls*{cpd} can be interpreted as a special case of the \gls*{mlsvd} where the core $\ten{S} \in \mathbb{R}^{R \times R \times R}$ is a cubical diagonal tensor with nonzero entries and the factor matrices are not orthogonal. The orthogonal \gls*{cpd} \citep{kolda_orthogonal_2001,sorensen_canonical_2012} retains the orthogonality of the factor matrices and is obtained from \cref{thm:linear_mlsvd} by choosing linear feature maps and a compatibility tensor such that the compatibility equation \cref{eq:compatibility} is satisfied.

\subsection{Kernel \gls*{svd}}
The \gls*{kmlsvd} generalizes the \gls*{ksvd} to higher-order tensors. Is is therefore straightforward to obtain the \gls*{ksvd} of a matrix $\mat{K}$ from \cref{thm:mlsvd} by considering two data sources $\mat{\Phi}_1$ and $\mat{\Phi}_2$ and a positive diagonal regularization matrix $\mat{S}$.
Said primal \gls*{ksvd} optimization problem coincides with the one identified by \citet{suykens_svd_2016}.
\Cref{thm:nonlinear_mlsvd} then results in the shifted eigenvalue problems
\begin{align*}
    \mat{U}_1\, \mat{S} &= \mat{K} \,\mat{U}_2, \\
    \mat{U}_2\, \mat{S}\trans &= \mat{K}\trans \,\mat{U}_1,
\end{align*}
where $\mat{K} = \mat{\Phi}_1 \mat{C} \mat{\Phi}_2\trans$ is the kernel matrix, which is not required to be symmetric or positive-definite. 
The solution of the shifted eigenvalue problem is then the \gls*{svd} of $\mat{K}$, i.e. $\mat{K} = \mat{U}_1\, \mat{S}\, \mat{U}_2\trans$ \citep{lanczos_linear_1958}.
Similarly, the \emph{linear} \gls*{svd} \citep{suykens_svd_2016} is also recovered by additionally considering as features rows and columns of the data matrix $\mat{X}\in\mathbb{R}^{N_1\times N_2}$ and as compatibility matrix $\mat{C}\in\mathbb{R}^{M_1\times M_2}$ as defined in \cref{sub:kernel_mlsvd}.
Notably in contrast with the \gls*{kmlsvd} case detailed in \cref{sec:linear_mlsvd}, in the $2$-dimensional \gls*{ksvd} case it is possible to easily define an asymmetric kernel function $\kappa(\mat{x}_1, \mat{C}\mat{x}_2)$ based on a predefined symmetric (positive-definite) kernel function $\kappa(\cdot,\cdot)$, exploiting the fact that the compatibility matrix maps one input space to the other, as proposed by \citet{tao_nonlinear_2023}.

\subsection{Kernel \gls*{pca}}
Consider a feature matrix $\mat{\Phi} \in \mathbb{R}^{N \times M}$. In the primal problem of \cref{def:primal_problem} we are now interested in finding two weight matrices $\mat{W}_1, \mat{W}_2$ that project the data to score variables $\mat{\Phi}\mat{W}_1, \mat{\Phi}\mat{W}_2$ with maximal variance. The compatibility matrix $\mat{C}$ is chosen to be a positive diagonal matrix, and the regularization parameter matrix $\mat{S}$ is chosen diagonal with positive entries. \Cref{thm:nonlinear_mlsvd} then results in the shifted eigenvalue problems
\begin{align*}
    \mat{U}_1 \mat{S} &= \mat{K} \mat{U}_2, \\
    \mat{U}_2 \mat{S} &= \mat{K} \mat{U}_1,
\end{align*}
where $\mat{K} = \mat{\Phi}\mat{C}\mat{\Phi}\trans$ is a symmetric positive-definite kernel matrix. 
The symmetry of the kernel matrix implies that $\mat{U}_1=\mat{U}_2$, ensuring that the dual problem is the eigenvalue decomposition of the kernel matrix i.e. $\mat{K} = \mat{U}\mat{S}\mat{U}\trans$.
Linear \gls*{pca} is recovered by choosing $\mat{\Phi} = \mat{X}$.

\subsection{Higher-Order \gls*{kpca}}

\Cref{thm:mlsvd} makes it possible to define a higher-order \gls*{kpca} which considers higher-order interactions between the (mapped) data. Consider e.g. three identical features $\mat{\Phi}_1=\mat{\Phi}_2=\mat{\Phi}_3 = \mat{\Phi}$ and akin to \gls*{kpca} a supersymmetric compatibility tensor $\ten{C}$ and positive superdiagonal regularization $\ten{S}$. Then by \cref{thm:mlsvd} the dual optimization problem is
\begin{align*}
    \mat{U}_1 \mat{S} &= \mat{K} \left(\mat{U}_3 \kron \mat{U}_2 \right),\\
    \mat{U}_2 \mat{S} &= \mat{K} \left(\mat{U}_3 \kron \mat{U}_1 \right),\\
    \mat{U_3} \mat{S} &= \mat{K} \left(\mat{U}_2 \kron \mat{U}_1 \right),
\end{align*}
where $\mat{K} = \mat{K}_{(1)} = \mat{K}_{(2)} \defis \mat{K}_{(3)}$ and $\mat{S}_{(1)} = \mat{S}_{(2)} = \mat{S}_{(3)} \defis \mat{S}$ by construction. Consequently by the orthogonality of $\mat{U}_1$, $\mat{U}_2$ and $\mat{U}_3$ it follows that $\mat{U}_1=\mat{U}_2=\mat{U}_3\defis \mat{U}$, resulting in $\mat{U} \mat{S} = \mat{K}(\mat{U}\kron \mat{U})$. 
Examples of kernel tensors that encode said higher-order interactions are the polynomial and exponential kernels in \cref{example:poly_exp_kernels} \citep{salzo_solving_2018-2,salzo_generalized_2020}.

\section{Related work}\label{sec:related_work}

As already mentioned, the development of a primal-dual formulation of \gls*{pca} was carried out by \citet{suykens_support_2003}, but the idea to kernelize the method is older and generally attributed to \citet{mika_kernel_1998}. 

The \gls*{svd} was cast in the primal-dual framework by \citet{suykens_svd_2016}, who proposed to kernelize the approach. This idea was recently further developed by \citet{tao_nonlinear_2023}, who proposed to use the compatibility matrix to build asymmetric kernels departing from standard symmetric and positive-definite kernels. \citet{tao_nonlinear_2023} also extended the Nyström method in order to efficiently approximate said asymmetric kernels, enabling to fully exploit the computational advantages that stem from the primal \gls*{ksvd} optimization problem.
\citet{chen_primal-attention_2023,chen_self-attention_2024} proposed to decompose the self-attention kernel matrix in \emph{Transformer} networks \citep{vaswani_attention_2017} using the primal formulation of the \gls*{ksvd}. This enables to fully capture its asymmetric nature in contrast with the existing alternatives, which consider only the row or column-space and are therefore effectively discarding information. 
\citet{he_learning_2023} considered asymmetric kernels in the \gls*{ls-svm} primal-dual formulation and consider applications in the context of directed graphs, where asymmetry is naturally present.  
\citet{he_random_2024} extended the celebrated \gls*{rff} \citep{rahimi_random_2007} kernel approximation framework to handle asymmetric kernels.

The \gls*{mlsvd} and the related Tucker decomposition \citep{tucker_implications_1963-1,tucker_mathematical_1966}, which relaxes the orthogonality constraint on the factor matrices, has applications in signal and image processing, computer vision, chemometrics, finance, human motion analysis, data mining, machine learning and deep learning. 
We redirect the interested reader to the survey papers of \citet{kolda_tensor_2009,cichocki_tensor_2016,cichocki_tensor_2017,panagakis_tensor_2021}.
Kernelizing the \gls*{mlsvd} was attempted by \citet{li_kernel-based_2005,zhao_kernelization_2013-1} who proposed to map each unfolding of a data tensor to the \emph{same} feature space, whose left-singular vectors are the factor matrices of the decomposition. The approach can be easily encompassed in \cref{def:primal_problem}, which is more general as it allows for asymmetry and permutational variance by selecting the core tensor appropriately.

\section{Conclusions, Further Research and Applications}\label{sec:conclusion}

In this paper we extend the Lanczos decomposition theorem to the tensor case, enabling us to cast the \gls*{mlsvd} as a primal-dual optimization problem. This allows for a straightforward nonlinear extension in terms of feature maps in the primal, whose associated dual optimization problem is then the \gls*{mlsvd} of a kernel tensor. Importantly, the presented optimization framework recovers as special cases both \gls*{svd} and \gls*{pca}.
Besides the nonlinear extension, the benefits of having the \gls*{mlsvd} defined in terms of a primal and a dual optimization problem are computational in nature. In particular, the practitioner can opt for solving the cheapest optimization problem given the circumstances. The primal formulation, as typically the case in kernel methods, is more convenient when dealing with large sample sizes and relatively small features, while the dual allows to tackle the case where the features are large w.r.t the number of samples, or even infinite-dimensional.

In our opinion, further research is needed, specifically focusing on the theory, analysis, and design of kernel functions with more than two vector inputs. 
This research could explore the design of generally applicable kernel functions, such as the e.g. the ubiquitous Gaussian kernel in the two input case, or be driven by specific applications. 
An interesting research direction would then also consist in investigating the possible decomposition or approximation of said kernels using feature maps and core tensors, in order to fully leverage the computational advantages of the primal formulation.
The proposed approach can in principle be utilized as a nonlinear extension of \gls*{mlsvd} and thus be employed whenever its linear counterpart is used, which means e.g. for feature extraction, signal analysis, image processing and computer vision. An notable area of possible application is deep learning, where the primal formulation of the \gls*{mlsvd} could be applied to approximate nonlinear kernel tensors which often arise. 
Two examples are the activated convolution kernel in \gls*{cnn}, which thus far has been trained decomposed in Tucker form before being activated, or generalizations of the self-attention mechanism in \emph{Transformer} having multiple attention vectors instead of only queries, values and keys.

\bibliographystyle{abbrvnat}
\bibliography{bibliography}

\appendix
\section{A Primal-dual formulation for the \gls*{mlsvd}}
\label{appendixA}
For completeness we present the primal \gls*{mlsvd} optimization problem for any tensor of order $D$.
\begin{definition}[Primal \gls*{mlsvd} optimization problem]\label{def:primal_problemD}
Given $D$ feature matrices $\mat{\Phi}_1 \in \mathbb{R}^{N_1 \times M_1},\mat{\Phi}_2\in\mathbb{R}^{N_2 \times M_2}, \ldots, \mat{\Phi}_D\in\mathbb{R}^{N_D \times M_D}$, a compatibility tensor $\ten{C}\in\mathbb{R}^{M_1\times M_2\times \cdots \times  M_D}$, and regularization parameters $\ten{S}\in\mathbb{R}^{R_1\times R_2\times \cdots \times R_D}$, we define the primal optimization problem as
    \begin{align*}\label{eq:primal_problem}
        &\max_{\mat{W}_1,\ldots,\mat{W}_D,\mat{E}_1,\ldots,\mat{E}_D} J(\mat{W}_1,\ldots,\mat{W}_D,\mat{E}_1,\ldots,\mat{E}_D) \isdef \\
                &  \frac{1}{2}\sum_{d=1}^D \Tr{\left(\mat{E}_d \, (\mat{S}_{(d)}\mat{S}_{(d)}\trans)^{-1}\, \mat{E}_d\trans \right)} - (D-1)\,\vectorize{\ten{C}}\trans\,\left(\mat{W}_D \kron \cdots \kron \mat{W}_1 \right)\, \vectorize{\ten{S}} \\
                & + \frac{D-2}{2} \vectorize{\ten{C}}^T \left(\mat{\Phi}_D\trans \mat{\Phi}_D \kron \cdots   \kron \mat{\Phi}_1\trans \mat{\Phi}_1\right) \vectorize{\ten{C}} \\
                \textrm{such that:} \ &  \mat{E}_1 = \mat{\Phi}_1 \, \mat{C}_{(1)} \, \left( \mat{W}_D \kron \cdots \kron \mat{W}_3 \kron \mat{W}_2 \right) \, \mat{S}_{(1)}\trans,\\
    & \mat{E}_2= \mat{\Phi}_2 \, \mat{C}_{(2)} \, \left( \mat{W}_D \kron \cdots  \kron \mat{W}_3 \kron \mat{W}_1 \right) \, \mat{S}_{(2)}\trans,\\
    & \quad \quad  \vdots \\
    & \mat{E}_D =  \mat{\Phi}_D \, \mat{C}_{(D)} \, \left( \mat{W}_{D-1} \kron \cdots \kron  \mat{W}_2 \kron \mat{W}_1 \right) \, \mat{S}_{(D)}\trans .
    \end{align*}    
\end{definition}
\Cref{def:primal_problemD} reduces to Definition~\ref{def:primal_problem} when $D=3$.
\begin{theorem}\label{thm:mlsvdD}
    The primal optimization problem of \cref{def:primal_problemD} is equivalent to the \gls*{mlsvd} of the tensor $\ten{K}\in\mathbb{C}^{N_1 \times N_2 \times \cdots \times N_D}$. The tensor $\ten{K}$ is defined as
    \begin{equation}
        \vectorize{\ten{K}} \isdef \left(\mat{\Phi}_D \kron \cdots \kron  \mat{\Phi}_2 \kron \mat{\Phi}_1 \right) \; \vectorize{\ten{C}} .
    \end{equation}
\end{theorem}
\begin{proof}
The corresponding Lagrangian for the optimization problem in~\ref{def:primal_problemD} has $D$ Lagrange multiplier matrices $\mat{U}_1,\ldots,\mat{U}_D$. We write out the \gls*{kkt} conditions for $\mat{W}_1,\mat{E}_1$ and $\mat{U}_1$, as the remaining conditions are similar.  
\begin{align*}
    &\frac{\partial \mathcal{L}}{\partial \mat{W}_1} = 
    \mat{0} \iff (D-1)\, {\mat{C}_{(1)}} (\mat{W}_D \kron \cdots  \kron \mat{W}_2) \, \mat{S}_{(1)}\trans = {\mat{C}_{(1)}}  \underbrace{\left( \mat{W}_D \kron \cdots \kron \mat{\Phi}_2\trans \mat{U}_2  + \cdots +  \mat{\Phi}_D\trans \mat{U}_D \kron \cdots \kron \mat{W}_2 \right)}_{(D-1) \textrm{ terms}} \mat{S}_{(1)}\trans. 
\end{align*}
The \gls*{kkt} conditions for the remaining weight matrices will have similar form and all equalities are trivially satisfied when $\mat{W}_d =\mat{0}$ or $\mat{W}_d=\mat{\Phi}_d\trans \mat{U}_d \; (1 \leq d \leq D)$.
Setting the partial derivative of the Lagrangian with respect to $\mat{E}_1$ to zero results in the condition
\begin{align*}
\frac{\partial \mathcal{L}}{\partial \mat{E}_1} = \mat{0} \iff \mat{E}_1 = \mat{U}_1 \mat{S}_{(1)} \mat{S}_{(1)}\trans,    
\end{align*}
and likewise for the other error matrices. 
Substitution of $\mat{W}_1$ and $\mat{E}_1$ into the constraint results in
\begin{align*}
  \mat{U}_1 \mat{S}_{(1)} \mat{S}_{(1)}\trans &=  \mat{\Phi}_1 \, \mat{C}_{(1)} \, \left( \mat{W}_D \kron \cdots \kron \mat{W}_3 \kron \mat{W}_2 \right) \, \mat{S}_{(1)}\trans,\\
  &= \mat{\Phi}_1 \, \mat{C}_{(1)} \, \left( \mat{\Phi}_D\trans  \kron \cdots \kron \mat{\Phi}_2\trans \right) \, \left(\mat{U}_D \kron \cdots \kron \mat{U}_2 \right) \mat{S}_{(1)}\trans,\\
  &= \mat{K}_{(1)}\, \left(\mat{U}_D \kron \cdots \kron \mat{U}_2 \right) \mat{S}_{(1)}\trans,
\end{align*}
which is the first equation of the generalized Lanczos theorem. A similar construction applies for the remaining equations.
\end{proof}

\begin{corollary}
The \gls*{mlsvd} solution of the dual problem in \cref{thm:mlsvdD} results in a zero objective function $(J=0)$ in the primal optimization problem of \cref{def:primal_problemD}.
\end{corollary}
\begin{proof}
Plugging in the \gls*{kkt} conditions for $\mat{E}_1,\ldots,\mat{E}_D$ and $\mat{W}_1,\ldots,\mat{W}_D$ into the primal objective function $J$ results in
\begin{align*}
&\frac{1}{2}\sum_{d=1}^D \Tr{\left(\mat{E}_d \, (\mat{S}_{(d)}\mat{S}_{(d)}\trans)^{-1}\, \mat{E}_d\trans \right)} - (D-1)\vectorize{\ten{C}}\trans\,\left(\mat{W}_D \kron \cdots \kron \mat{W}_2 \kron \mat{W}_1 \right)\, \vectorize{\ten{S}} + \frac{(D-2)}{2} \vectorize{\ten{K}}^T \vectorize{\ten{K}}  \\
=&\frac{1}{2}\sum_{d=1}^D \Tr{\left(\mat{U}_d \, \mat{S}_{(d)}\mat{S}_{(d)}\trans\, \mat{U}_d\trans \right)} - (D-1)\vectorize{\ten{C}}\trans\,\left(\mat{\Phi}_D\trans \mat{U}_D \kron \cdots \kron \mat{\Phi}_2\trans \mat{U}_2 \kron \mat{\Phi}_1\trans \mat{U}_1 \right)\, \vectorize{\ten{S}}+ \frac{1}{2} \vectorize{\ten{K}}^T \vectorize{\ten{K}} \\
=&\left( \frac{D}{2} - (D-1) + \frac{(D-2)}{2} \right) \vectorize{\ten{K}}^T \vectorize{\ten{K}} =0. \\
\end{align*}
\end{proof}
\end{document}